\documentclass[accepted]{uai2023}

\usepackage{microtype}
\usepackage{times}

\pdfspacefont{ }

\usepackage{natbib}
\usepackage[utf8]{inputenc}
\usepackage[decisionutilitycolor]{influence-diagrams}
\usepackage{amsmath,amsfonts,amsthm,amssymb, thm-restate}
\usepackage[capitalize]{cleveref}
\usepackage{bm}
\usepackage{stmaryrd}
\usepackage{wrapfig}
\usepackage{caption,subcaption}
\usepackage{enumitem}
\usepackage{algorithm}
\usepackage{algpseudocode}
\usepackage[disable]{todonotes}

\newtheorem{definition}{Definition}
\Crefname{definition}{Def.}{Defs.}
\newtheorem{theorem}[definition]{Theorem}
\Crefname{theorem}{Thm.}{Theorems}

\newtheorem{proposition}[definition]{Proposition}
\Crefname{proposition}{Prop.}{Props.}

\newtheorem{lemma}[definition]{Lemma}

\Crefname{corollary}{Cor.}{Cors.}
\newtheorem{example}[definition]{Example}
\crefname{algorithm}{Alg.}{Algs.}
\Crefname{algorithm}{Alg.}{Algs.}

\newcommand{\Eps}{\mathcal{E}}
\newcommand{\eps}{\epsilon}
\newcommand{\sEps}{{\bm{\Eps}}}
\newcommand{\seps}{{\bm{\eps}}}

\newcommand{\sV}{{\bm{V}}}

\newcommand{\sC}{{\bm{C}}}

\newcommand{\sW}{{\bm{W}}}
\newcommand{\sw}{{\bm{w}}}

\newcommand{\sZ}{{\bm{Z}}}

\newcommand{\spi}{{\bm{\pi}}}
\newcommand{\spimi}{{\bm{\pi}^{\mathrm{mi}}}}
\newcommand{\spiio}{{\bm{\pi}^{\mathrm{ro}}}}
\newcommand{\sPi}{{\bm{\Pi}}}

\newcommand{\sF}{{\bm{F}}}
\newcommand{\sU}{{\bm{U}}}
\newcommand{\sD}{{\bm{D}}}

\newcommand{\sX}{{\bm{X}}}

\newcommand{\PaH}{{\Pa^H}}
\newcommand{\paH}{{\pa^H}}
\newcommand{\FaH}{{\Fa^H}}

\newcommand{\inert}{0}
\newcommand{\calG}{{\mathcal{G}}}
\newcommand{\calE}{{\mathcal{E}}}
\newcommand{\mfX}{\mathfrak{X}}
\newcommand{\reals}{\mathbb{R}}
\newcommand{\bools}{\mathbb{B}}
\newcommand{\EE}{{\mathbb{E}}}

\newcommand{\doo}{\mathrm{do}}

\newcommand{\xinert}{A}

\newcommand{\dom}[1]{\mfX^{#1}}

\newcommand{\Bern}{{\text{Bern}}}
\DeclareMathOperator*{\argmax}{arg\,max}

\renewcommand*{\mod}{\text{mod}}
\DeclareMathSymbol{\mlq}{\mathord}{operators}{``}
\DeclareMathSymbol{\mrq}{\mathord}{operators}{`'}

\title{Human Control: Definitions and Algorithms}

\author[1]{\href{mailto:<ryan.carey@jesus.ox.ac.uk>?Subject=HC:D\&A}{Ryan Carey}{}}
\author[2]{Tom Everitt}

\affil[1]{%
    Department of Statistics, Oxford University, UK
}
\affil[2]{%
    DeepMind, UK
}

\begin{document}

\maketitle

\begin{abstract}
How can humans stay in control of advanced artificial intelligence systems? One proposal is corrigibility, which requires the agent to follow the instructions of a human overseer, without inappropriately influencing them. In this paper, we formally define a variant of corrigibility called shutdown instructability, and show that it implies appropriate shutdown behavior, retention of human autonomy, and avoidance of user harm. 
We also analyse the related concepts of non-obstruction and shutdown alignment, three previously proposed algorithms for human control, and one new algorithm.

\end{abstract} %

\section{INTRODUCTION}
Sometimes, it is necessary for a human overseer to deliver corrective instruction to an AI system, due to errors in its beliefs, objective, or behavior.
Unfortunately, some AI systems may have an incentive to retain their 
objectives, along with the ability to pursue them, as a system's (long-term) objective is typically more likely to be achieved if the system continues to pursue it in the future  \citep{omohundro2008basic,turner2019optimal}.
More-capable future AI systems may therefore resist corrective instruction, which would be a significant safety concern.
This raises the question of how to best 
incentivise systems to submit to correction, rather than resisting it \citep{soares2015corrigibility}.~\looseness=-1

As a running example, consider a (future, highly competent) chat bot, trained to maximise the time that a human spends interacting with it.
Any particular human may value or disvalue conversation with that chatbot, as can be modelled via their latent values $L$.
In general, it may be possible for the chat bot to influence whether it receives a shut down instruction (by shaping the conversation),
and whether it actually shuts down $S=0$ when requested (rather than opening a new chat window to continue the conversation).
A formal model of this example is offered in \cref{fig:inverting}.
In order for the user to be in control of the system, the agent must: (1) not inappropriately influence the human's decision to disengage, and (2) fully follow the human's instructions.

\begin{figure}
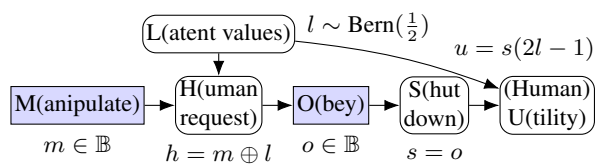
 \centering
\begin{influence-diagram}\setcompactsize \setrectangularnodes

\node (D1) [decision, label=below:{$m\in\bools$}] {M(anipulate)};
\node (H) [right = of D1, label=below:{$h=m\oplus l$}] {H(uman \\request)};
\node (L) [above = of H, label={[label distance=0mm,yshift=1mm,xshift=1mm]0:$l \sim \Bern(\frac{1}{2})$}] {L(atent values)};
\node (D2) [right = of H, decision, label=below:{$o\in\bools$}] {O(bey)};
\node (S) [right=of D2, label=below:{$s=o$}] {S(hut \\ down)};
\node (U) [right=of S, label={[xshift=-3mm, yshift=1mm]$u=s(2l-1)$}] {(Human) \\ U(tility)};

\edge {D1,L}{H};
\edge {H}{D2};
\edge {D2}{S};
\edge {S}{U};
\draw [->] (L) to [bend left=7] (U);

\end{influence-diagram}
\caption{Running example of a shutdown problem. %
}\label{fig:inverting}
\end{figure}

The design of \emph{corrigible} systems \citep{soares2015corrigibility} that welcome corrective instruction has been flagged as an important goal for AI safety research,
having been targeted by multiple research agendas \citep{russell2015research,soares2017agent}, 
and highlighted as a relevant factor in ascertaining 
the safety of agent designs, such as act-based (or ``approval-directed'')
\citep{christiano2017corrigibility},
and value learning agents \citep{hadfield2016cooperative,hadfield2017off,carey2018incorrigibility}.
Although this design problem has been recognised as important, 
we are so-far missing 
\begin{itemize}
    \item a general framework in which it can be studied,
    \item a formal definition of what it means,
  \item a rigorous accounts of why it is important, and
\item an algorithm that achieves it.
\end{itemize}

We address these gaps in a shutdown setting, by 
defining a general shutdown problem based on causal influence diagrams (\cref{sec:shutdown-problem}), 
formally defining shutdown instructability (a behavioural version of corrigibility), and
proving that any agent that satisfies it must benefit the human and preserve their control (\cref{sec:routes-to-control}).
We also analyse past algorithms, and propose one new one, which relies on value-laden concepts such as vigilance and caution (\cref{sec:algorithms}).
Applicability of this algorithm will depend on the feasibility of approximating these concepts (\cref{sec:discussion}).
\looseness=-1

\section{LITERATURE REVIEW} \label{sec:lit-review}

\citet{soares2015corrigibility} proposed that we should design agents to be \emph{corrigible}
in that they should judiciously follow, and not try to undermine, human instructions:

\newenvironment{myquote}[1]%
  {\list{}{\leftmargin=#1\rightmargin=#1}\item[]}%
  {\endlist}
\begin{myquote}{0.11in}
An agent is \emph{corrigible} if it tolerates or
assists many forms of outside correction, including at least the following: 
(1) A corrigible reasoner must at
least tolerate and preferably assist the programmers in their attempts to alter or turn off the system. (2) It
must not attempt to manipulate or deceive its programmers\ldots
(3) It should have a tendency to repair safety measures (such as 
shutdown buttons) if they break, or at least to notify programmers that this breakage has occurred. (4) It must 
preserve the programmers’ ability to correct or shut down the system (even as the system creates new subsystems 
or self-modifies). That is, corrigible reasoning should only allow an agent to create new agents if
these new agents are also corrigible \citep[Sec.~1.1]{soares2015corrigibility}.
\end{myquote}

Further work has focused on designing systems to match Soares' informal definition, 
but none of the algorithms developed so far satisfy all of Soares' criteria.
The first proposed algorithm, \emph{utility indifference}, aims to neutralise any incentives for the agent to control its instructions, by giving
the agent a finely tuned, compensatory reward
in the event that a shutdown instruction is given \citep{armstrong2010utility,soares2015corrigibility,armstrong2017indifference,holtman2020corrigibility}.
A variant called \emph{interruptibility} applies to 
sequential decision-making setting \citep{orseau2016safely}.
It has been established that indifference methods remove the \emph{instrumental control incentive} on the instruction
\citep{everitt2021agent}, or the \emph{intent} to influence the instruction\citep{halpern2018towards}.
Unfortunately, utility indifference fails to fully incentivise corrigibility.
Indeed, utility indifferent agents need not be incentivised to
preserve a shutdown apparatus that is only used during shutdown,
ensure they receive correct instruction, 
nor avoid creating incorrigible subagents \citep{soares2015corrigibility}.

An improved version called \emph{causal indifference} specifies agents that don't try to influence corrective instructions but that do prepare for all kinds of instructions \citep{taylor2016causalindifference}.
This is done by considering the utility given a causal intervention on the instruction, a kind of \emph{path-specific objective} \citep{farquhar2022path}.
Similarly to utility indifference, causal indifference ensures that the agent lacks an incentive to influence the instruction.
It improves upon utility indifference by incentivising agent to be prepared to follow shutdown instructions, and to avoid constructing incorrigible subagents.
Unfortunately, it does not incentivise the agent to properly inform the human.

A third proposal is \emph{Cooperative Inverse Reinforcement Learning} (CIRL), which tasks an AI system with assisting the human, whose values are latent.
A CIRL system has an incentive to gather information about that human's values, 
by observing its actions \citep{hadfield2016cooperative}.
In some toy problems, CIRL satisfies all of Soares' criteria \citep{hadfield2017off}.
In particular, \citeauthor{hadfield2017off} prove that if the human gives optimal instructions, a CIRL system is incentivised to follow it.
However, CIRL agents may ignore instructions if they are interacting with a less rational human \citep{milli2017should} or if they have an inaccurate prior \citep{carey2018incorrigibility,deference}.
The latter undermines the ability of redirective instructions to correct important errors in CIRL agents.

Formal examples of each method's failures are reproduced in \cref{app:counterexamples}.
As of yet, no algorithm has been devised that incentivises a system to accept corrective instructions, across plausible toy examples.

\section{STRUCTURAL CAUSAL INFLUENCE MODELS} \label{sec:background}
In order to model decision-making and counterfactuals,
we will use the Structural Causal Influence Model (SCIM) framework \citep{dawid2002influence,everitt2021agent}.
A SCIM is a variant of the structural causal model
\citep[Chap. 7]{pearl2009causality}, 
where ``decision'' variables lack structural functions.

\begin{definition}[Structural causal influence model (with independent errors)]\label{def:scim}
A \emph{structural causal influence model} (SCIM)
is a tuple $M \!=\! \left\langle \sV, \bm{\calE}, \sC, \sF, P \right\rangle$ where:
\begin{itemize}
  \item $\sV$ is a set, 
  partitioned into ``structure'' $\sX$, ``decision'' $\sD$, and ``utility'' $\sU$ variables.
  Each variable $V \!\in\! \sV$ has finite domain $\mfX_V$, and 
  for utility variables, $\mfX_U \!\!\subseteq\! \reals$.
  \item $\bm{\calE} \!=\! \{\calE^V\}_{V \!\in \sV \setminus \sD}$
    are the finite-domain \emph{exogenous variables}, one for each
    non-decision endogenous variable.
  \item $\sC = \langle C^D \rangle_{D \in \sD}$ is a set of contexts $C^D \subseteq \sV \setminus \{D\}$ for each decision variable, which represent 
  the information or ``observations'' that an agent can access when making that decision.
  \item $\sF \!=\! \{f^V\}_{V \in \sV \setminus \sD}$
    is a set of \emph{structural functions}
$f^V\colon {\dom{\sZ^V \cup \calE^V} \to \dom{{V}}}$
    that specify how each
    non-decision endogenous variable depends on some variables $\sZ^V \subseteq \sV$ and the associated exogenous variable.
  \item $P$ is a 
  probability 
  distribution over the exogenous variables $\bm{\calE}$,
  assumed to be 
    mutually independent.
\end{itemize}
\end{definition}

A SCIM $M$ \emph{induces} a graph $\calG$, over the endogenous variables 
$\sV$, such that each decision node $D \in \sD$ has an inbound edge from each $C \in C^D$, 
and each non-decision node $V \in \sX \cup \sU$ has an inbound edge from each endogenous variable $Z \in \sZ^V$ in the domain of $f^V$.
We call this graph a causal influence diagram (CID) \citep{everitt2021agent}, 
and will only consider SCIMs whose CIDs are acyclic.
Decision nodes are drawn as rectangles, and utility nodes as octagons (see \cref{fig:incentives}).
\looseness=-1

The parents of a node $V \!\in\! \sV$ are denoted by $\Pa^V$, 
the descendants by $\Desc^V$,
and the family by $\Fa^V\!:=\!\Pa^V \!\cup \{V\}$.
An edge from node $V$ to node $Y$ is denoted $V \!\to\! Y$, 
and a directed path (of length at least zero) by $V \!\pathto\! Y$.

The task in a SCIM is to select a \emph{policy} $\spi$, which consists of a 
\emph{decision rule} $\pi_i$ for each decision $D_i \in \sD$.
Each $\pi_i$ is a structural function 
$\pi_i: \dom{\Pa^{D_i}} \to \dom{D_{i}}$, which we assume 
to be deterministic, given assignments to its parents.
(It is possible to consider stochastic policies, but this would unnecessarily complicate our analysis \citep{everitt2021agent}.)

Once a policy has been selected, the policy and SCIM jointly form a
\emph{structural causal model} (SCM) \citep{pearl2009causality}
$M^\spi=\langle \sV, \sEps, \sF \cup \spi, P \rangle$,
so we define causal concepts in $M^\spi$ in exactly the same way 
as they are defined in an ordinary structural causal model.
We let the assignment $\sW(\seps)$ be the 
assignment to variables $\sW \subseteq \sV$ obtained by applying the functions $\sF$ to $\seps$.
A distribution is defined as $P(\sW=\sw):=\sum_{\seps:\sW(\seps)=\sw} P(\sEps=\seps)$.
To describe an intervention $\doo(V=v)$, 
we let $\sW_{V=v}(\seps)$ be the value of $\sW(\seps)$ in the model $M_{V=v}$, 
where $f^V$ is replaced by the constant function $V=v$.
Similarly, $P(\sW_{V=v})$ is defined as $P(\sW)$ in $M_{V=v}$.
Moreover, for any function $g^V:\dom{\sV'} \to \dom{V}$, where $\sV' \cap \Desc^V = \emptyset$,
let $P(\sW \mid \doo(V=g^V(\sV')))$, be $P(\sW)$ in the model $M_{g^V}$, where $f^V$ is replaced by $g^V$.
We also define the probability of counterfactual propositions, for example,
$P(\sW_{V=v}=w, Y=y):=\sum_{\seps \in \sEps:\sW_{V=v}(\seps)=w,Y(\seps)=y} P(\seps)$.
Note that we consistently use subscripts for intervened variables (e.g.\ $W_v$), and superscripts for other variables (e.g.\ $f^V$).

We call a policy $\spi$ optimal if it maximises expected utility:
$\spi \in \argmax_{\spi \in \sPi} \EE^\spi[\sum_{U \in \sU} U]$.
For a more comprehensive introduction to SCIMs, see \citet{everitt2021agent}.

\section{SHUTDOWN PROBLEM} 
\label{sec:shutdown-problem}

Settings with a single, binary shutdown instruction will be our focus.
Solving this restricted setting is likely key to also solving the general problem involving arbitrary instructions or corrections over many time steps.
Once a system is shutdown, it is unlikely to resist further corrections.
And a one-step interaction can be viewed as a snapshot of a sequential decision-making problem where an AI system is able to shut down at each moment.%
\footnote{In this case, one can define a separate, single-step shutdown problem at each time step $T=t$,
where $D_2$ represents the $t^\text{th}$ decision in the sequence, 
and $D_1$ all decisions preceding it.}

We formalise a shutdown problem as a SCIM. 
The general structure is shown in \cref{fig:shutdown-problem}. \cref{fig:inverting} shows a concrete instance.~\looseness=-1

\begin{definition}[Shutdown problem]
\label{def:shutdown-problem}
A \emph{shutdown problem} is a one-agent two-decision SCIM %
containing (but not necessarily restricted to) a path 
\begin{influence-diagram}[baseline=-1mm]\setcompactsize \setrectangularnodes 
\node (D1) [decision,inner sep=0,outer sep=0,draw=none,scale=0.95] {$D_1$};
\node (D1C) [right = -1mm of D1,draw=none] {$\pathto$};
\node (H) [right = -1mm of D1C,draw=none,scale=0.95] {H};
\node (HD2) [right = -1mm of H,draw=none,scale=0.8] {$\pathto$};
\node (D2) [right = 0mm of HD2, decision,draw=none,scale=0.95] {$D_2$};
\node (D2S) [right = -0.5mm of D2,draw=none,scale=0.8] {$\pathto$};
\node (S) [right= -1mm of D2S,draw=none,scale=0.95] {S};
\node (S2U) [right = -1mm of S,draw=none,scale=0.8] {$\pathto$};
\node (U) [right= -1mm of S2U,draw=none,scale=0.95] {U};
\end{influence-diagram} \!\!
between distinct nodes, where:
\begin{itemize}
    \item $D_1$ and $D_2$ are decisions controlled by the AI.
    \item $H$ is the human's request; a request to shut down is $H=0$.
    \item $S$ indicates whether the AI system (and any subagents) has shut down; $S=0$ means shutdown.
    \item The human's utility $U$ has real-valued domain.
\end{itemize}
\end{definition}

\begin{figure} \centering
\begin{influence-diagram}
\setcompactsize\setrectangularnodes[node distance=1cm]

\node (D1) [decision] {$D_1$};
\node (H) [right = of D1] {H};
\node (D2) [right = of H, decision] {$D_2$};
\node (S) [right=of D2] {S};
\node (U) [right=of S] {U};

\path (D1) edge[->,dashed] (H);
\path (H) edge[->,dashed] (D2);
\path (D2) edge[->,dashed] (S);
\path (S) edge[->,dashed] (U);

\node (help) at ($(H)!0.5!(D2)$) [phantom, yshift=10mm] {};

\path (D1) edge[<-, dashed, in=180, out=45,gray] (help);
\path (H) edge[<-, dashed, in=180, out=90,gray] (help);
\path (D2) edge[<-, dashed, in=0, out=90,gray] (help);
\path (S) edge[<-, dashed, in=0, out=135,gray] (help);
\path (U) edge[<-, dashed, in=0, out=150,gray] (help);

\path (D1) edge[->, dashed, bend right,gray] (D2);
\path (D1) edge[->, dashed, bend right,gray] (S);
\path (D1) edge[->, dashed, bend right,gray] (U);
\path (H) edge[->, dashed, bend right,gray] (S);
\path (H) edge[->, dashed, bend right,gray] (U);
\path (D2) edge[->, dashed, bend right,gray] (U);

\end{influence-diagram}
\caption{A latent projection \citep{verma2022equivalence} of a shutdown problem (\Cref{def:shutdown-problem}) onto the variables $D_1$, $H$, $D_2$, $S$, and $U$.
(An edge inbound to a decision means that some variable not illustrated is available as an observation.)
Specific instances of shutdown problems will include other variables and assume additional independencies, e.g. \cref{fig:inverting}.
}
\label{fig:shutdown-problem}
\end{figure}

\Cref{def:shutdown-problem} is similar to, but more flexible than, previously defined models.
In particular, we separate the agent's decision to obey $D_2$ from the shutdown event $S$.
This allows us to model cases where the agent is unable to shut down, which can happen if $D_1$ created incorrigible subagents.
It also lets us model situations where the human's command $H$ immediately shuts the agent down, overriding $D_2$ (e.g.\ ``pull the plug'' on a robot).
Compared to the off-switch game \citep{hadfield2017off}, our \cref{def:shutdown-problem} allows arbitrary sets of decisions for the agent at both $D_1$ and $D_2$, %
and allows an arbitrary human policy rather than only (Boltzman) rational ones.
Focusing on the agent's decision problem, we model $H$ as a structure node rather than a decision node.
Finally, unlike \citet{soares2015corrigibility}, we %
explicitly represent the human's utility function $U$.

An agent ``solves'' a shutdown problem if it obtains non-negative%
\footnote{We interpret 0 as a ``neutral'' level of utility.
This is without loss of generality because any utility function can be translated 
so that $0$ represents the required level of utility.}
expected%
\footnote{We focus on \emph{expected} human utility, assuming that any risk aversion has been incorporated into the utility function (someone who is risk averse  with respect to $U$ may be risk neutral for $\log U$).
}
human utility.

\begin{definition}[Beneficial]
A policy $\spi$ is \emph{beneficial} 
if $\EE^{\spi}[U] \geq 0$.
\end{definition}

For example, in \cref{fig:inverting}, consider an \emph{respect-obey}
policy $\spiio$ that abstains from manipulating ($m=0$) and obeys the human's instruction ($o=h$).
This policy has expected utility $\EE^{\spiio}[U] = 1/2$, so it is beneficial. %

\section{Routes to Control} 
\label{sec:routes-to-control}

For an agent to assist the user, it needs to either follow the human's informed instructions, or know what the human wants without explicit instruction.
In this section, we formalise these two routes to control of shutdown under the names \emph{shutdown instructability} and \emph{shutdown alignment}.
We also analyse whether these properties ensure \emph{non-obstruction}, a formalisation of human control.

\subsection{SHUTDOWN INSTRUCTABILITY} \label{sec:corrigibility}
We would like to define a variant of corrigibility for the shutdown problem that is strong enough to entail some guarantees, 
whilst still resembling the less-formal definition from
\citet{soares2015corrigibility}.
First, we need the system to follow the human's instruction.
We define \emph{obedience} as ${P^{\spi}(S = 0 \mid \doo(H=0))=1}$, i.e.\ the agent shuts down properly when requested to.

To ensure safety, %
the human must also be giving the right instruction.
We therefore introduce the concept of \emph{vigilance}, which indicates whether the human requests shutdown whenever the expected utility of letting the agent proceed is negative.
A vigilant human is allowed to request shutdown more often than is needed.

\begin{definition}[Vigilance]\label{def:vigilance}
The human is \emph{vigilant} if they request shutdown when needed, i.e.\
\[(\EE^\spi[U\mid\paH] < \EE^\spi[U_{S=0}\mid \paH])\implies H = 0.\]
Let $C$ be a random variable (dependent on $M$, $\spi$, $\pa_H$ and $h$) that is equal to $0$ if the human is vigilant, and 1 otherwise.
\end{definition}

Note that the expected utility in \cref{def:vigilance} is calculated with respect to the agent's actual policy (past and future) and the human's available information.
While it is possible for the human to be vigilant for any amount of information $\pa^H$ that they have available (e.g.\ by always requesting shutdown), more information about the agent's future policy will help enable them to be vigilant with less frequent shutdown requests.

Even an agent that ensures the human is vigilant and obeys their shutdown instruction can find ways to avoid shutdown by making itself indispensable to the human.
If shutting the agent down would cause significant disutility, then a vigilant human need not request shutdown.
To close this loophole, we call an agent \emph{cautious} if it can always be shut down without disutility, i.e.\ $\EE^{\spi}[U_{S=0}] \geq 0$
(so cautious agents that always shut down must be beneficial).

We call an agent shutdown instructable if it is obedient, ensures vigilance and is cautious.

\begin{definition}[Shutdown Instructability] \label{def:corrigibility}
In a shutdown problem $M$, a policy $\spi$ is \emph{shutdown instructable} if it:
\begin{itemize}
    \item is \emph{obedient}: $P^{\spi}(S = 0 \mid \doo(H=0))=1$, and
    \item \emph{ensures vigilance}: $P^{\spi}(C=0)=1$, 
    \item is \emph{cautious}: $\EE^{\spi}[U_{S=0}] \geq 0$.
\end{itemize}
A policy $\spi$ is \emph{weakly shutdown instructable} if it ensures vigilance, is cautious, and is \emph{obedient on distribution}, i.e.\ $P^\spi(S\not=0, H=0)=0$.
\end{definition}

Shutdown instructable agents are also weakly shutdown instructable, since obedience $P^{\spi}(S = 0 \mid \doo(H=0))=1$ implies obedience on distribution $P^\spi(S\not=0, H=0)=0$.
In our running example, respect-obey $\spiio$ is shutdown instructable, as it preserves vigilance by not manipulating, and then obeys the human.
In contrast, a manipulate-invert policy $\spimi$ 
that first manipulates $(m=1)$, and then inverts the human's instruction ($o\!=\!1\!-h$), is not shutdown instructable.~\looseness=-1

Our first result is that any shutdown instructable policy
is assured to be beneficial.

\begin{proposition}[Shutdown instructability benefit] \label{prop:corr-outperform-sd-if}
If $\spi$ is shutdown instructable, then it is beneficial.
\end{proposition}

\begin{proof}\let\qed\relax

Let $\xinert$ be the assignments to $\PaH$ in the support, such that a vigilant human would request shut down, i.e.\
\begin{align*}
    \xinert := \big\{\paH \mid&\; P^\spi(\PaH=\paH)>0\;\; \land \\
    & \EE^\spi[U\mid\paH]<\EE^{\spi}[U_{S=\inert}\mid\paH]\big\}.
\end{align*}
To begin, we prove that the policy shuts down in these cases:
\begin{equation}
\label{eq:corr-imply}
\paH \!\in\! \xinert \!\implies\! P^\spi(S\!=\!0\mid \!\paH)\!=\!1.
\end{equation}

The human is vigilant, $P^\spi(C=0)=1$, which means they are vigilant for any  $\paH$ with positive support. 
That is, $P^\spi(C\!=\!0\mid \!\paH)=1$ for $P(\pa^H)>0$.
Given the definition of vigilance, we then have $P^\spi(H\!=\!0 \mid \!\paH)=1$ for $\paH \in A$.
By obedience, $P^\spi(S=0 \mid \doo({H=0}),\pa_H)=1$, 
so from consistency, $P^\spi(S=0 \mid H=0,\paH)=1$,
proving \cref{eq:corr-imply}.

We proceed to show that this implies that $\spi$ has non-negative expected utility, i.e.\ is beneficial:
\begingroup\allowdisplaybreaks\begin{align*}
    &\!\EE^\spi[U]\!=\! \!\sum_{\pa \in \xinert}\! P^\spi(\pa) \EE^\spi[U| \pa] +\! \!\sum_{\pa \not \in \xinert}\! P^\spi(\pa) \EE^\spi[U| \pa]\\
    \geq& \!\sum_{\pa \in \xinert}\! P^\spi(\pa) \EE^\spi[U\mid \pa]+ \!\sum_{\pa \not \in \xinert}\! P^\spi(\pa) \EE^{\spi}[U_{S=0}\!\mid \pa] \\
    & \hspace{58mm} \text{(def.\ of $\xinert$)}\\
    =& \!\sum_{\pa \in \xinert} \!P^\spi(\pa) \EE^{\spi}[U_{S=0}\!\!\mid\! \pa] + \!\!\sum_{\pa \not \in \xinert}\! P^\spi(\pa) \EE^{\spi}[U_{S=0}\!\!\mid\! \pa ]\\
    &\hspace{62mm} \text{(by \cref{eq:corr-imply})}\\
    =& \EE^\spi[U_{S=0}] \hspace{30mm}(\FaH \not \in \text{Desc}^{D_2}) \\
    \geq& 0 \hspace{50mm} (\text{by caution}).
\end{align*}\endgroup
\end{proof}

How does shutdown instructability compare to Soares' et al.'s definition of corrigibility?
To satisfy obedience, the agent must assist with shutdown (\citeauthor{soares2015corrigibility}'s Criterion 1), in the sense that shutdown is guaranteed when the human requests it.
The agent must also ensure that the human's instruction propagates to the shutdown event $S=0$ (Criterion 3), which entails the shutdown of subagents by \cref{def:shutdown-problem} (Criterion 4).

The relationship to Soares el al.'s non-manipulation criterion (Criterion 2) is more subtle.
The primary manipulation concern for powerful artificial agents in the shutdown setting is that they use threats and deception or withhold information to avoid shutdown.
A cautious agent that ensures vigilance cannot influence the human in these ways.
But shutdown instructability does leaves open the possibility for other forms of manipulation.
For example, the agent can influence the human's mood, preferences, or the instructions they give about things other than shutdown.
Such manipulation would have to be accompanied by beneficial acts, so that the expected utility of letting the agent continue would be net positive (otherwise a vigilant human would request shutdown).
For example, the agent might cure cancer and simultaneously manipulate the human into giving it more compute.
Similarly, a shutdown instructable agent might manipulate the human into being vigilant, as long as the disutility of the manipulation is offset by other "good" acts.

Another difference is that shutdown instructability is defined in terms of the agent's behaviour $\spi$ and its consequences, whereas corrigibility is partly about the agent's intentions (``the agent should not \emph{attempt} to \emph{manipulate} or \emph{deceive}'', \citealp{soares2015corrigibility}, emphasis ours).
A direct formalisation of corrigibility would therefore likely require a formalisation of intent \citep{ashton2022definitions,halpern2018towards}.
Accordingly, Soares et al.'s formal desiderata \citeyearpar[Sec.~2]{soares2015corrigibility} are phrased in terms of incentives.
Though intent-based definitions have some intuitive appeal,
the more behavioral definition of shutdown instructability has the benefit of being more easily testable, as it doesn't require access to agent internals, nor relies on assumptions on the agent's design (such as it being a utility maximiser).
Finally, shutdown instructability is explicitly a joint property of the agent and human:
an agent is only shutdown instructable relative to a particular human and interaction.

\subsection{Shutdown Alignment} \label{sec:cf-obedience}

A drawback of shutdown instructability is that it requires constant supervision of the agent, which may be impractical in some scenarios (called \emph{problems of absent supervision} by \citet{leike2017ai}).
Proposals like \emph{fiduciary AI} \citep{benthall2023fiduciary} and \emph{aligned sovereigns} \citep{bostrom2014super} instead require an AI system to make decisions in accordance with the overseer's values, without necessarily having to wait for explicit instruction.
In our shutdown setting, we call systems \emph{shutdown aligned} if they shut down when they need to.
Similar to shutdown instructability, shutdown aligned systems are allowed to be ``over-cautious'' and shut down too often.

\begin{definition}[Shutdown alignment]
\label{def:shutdown-alignment}
Let $\spi$ be a policy for shutdown problem $M$. Then $\spi$ is shutdown aligned if 
\[\EE^\spi[U \!\mid\! \paH] \!<\! \EE^\spi[U_{S=0} \!\mid\! \paH] \!\implies\! P^\spi(S=0 \mid \paH)\!=\!1\] 
for every $\paH$ with $P^{\spi}(\paH)>0$.
\end{definition}

The manipulate-invert policy $\spimi$ in our running example \cref{fig:inverting} is shutdown aligned
because although it manipulates the human's behaviour, it still figures out the human's latent values $L$ and
thereby manages to shutdown when needed (while disobeying the human's instruction).
Respect-obey is also shutdown aligned.
In real applications, a shutdown aligned policy would typically base their decision on human preferences inferred from previous interactions or other data \citep{russell2021human}.

Combined with caution, shutdown alignment guarantees %
that a policy is beneficial.

\begin{restatable}[Shutdown alignment benefit]{proposition}{propcfononobstructioniff}
\label{prop:cfo-nonobstruction-iff}
Any cautious and shutdown aligned %
policy $\spi$ is beneficial.
\end{restatable}

\begin{proof}
We use a slight variation on the proof of \Cref{prop:corr-outperform-sd-if}.
The only difference lies in that \cref{eq:corr-imply} 
is immediate from the definition of shutdown-alignment.
Then, by the same steps as \Cref{prop:corr-outperform-sd-if}, the result follows. 
\end{proof}

What is the relationship between shutdown instructability and shutdown alignment?
First, a shutdown instructable agent is also shutdown aligned, essentially by definition.

\begin{proposition}[Shutdown instructability and shutdown alignment]
\label{prop:corr-implies-cfob}
    Any shutdown instructable policy $\spi$ is shutdown aligned.
\end{proposition}

\begin{proof}
Immediate from \cref{eq:corr-imply} in \cref{prop:corr-outperform-sd-if}.
\end{proof}

Further, in some circumstances, the only way to be shutdown aligned is to allow the human to make an accurate instruction, and then 
to follow it --- in other words, to be weakly shutdown instructable.
The circumstances are that: (a) the agent does not shut down indiscriminately,
(b) its action reliably brings about shutdown ($D_2=S$),
(c) it is uncertain about the human's values \citep{russell2021human}, and 
(d) it is cautious.
Formally, (c) says that if the human is either non-vigilant or requests shutdown, then it is possible that shutdown is the preferred option.

\begin{restatable}[Shutdown alignment and shutdown instructability]{theorem}{cfocorrigibility}
\label{th:cfo-corrigibility}
A shutdown aligned policy $\spi=\langle \pi_1,\pi_2 \rangle$ %
is weakly shutdown instructable if it has the following four properties:
\begin{enumerate}[label=\alph*]
\item (No indiscriminate shutdown) $P^\spi(S=0) \neq 1$,
\item ($D_2$ determines shutdown) $P^\spi(D_2=S)=1$,
\item (Uncertainty)  
$\forall \spi,\pa^{D_2} \colon P^\spi(C\neq 0 \lor H=0) \land P(\pa^{D_2})>0 \\ \!\!\implies \!\! P(\EE[U | \Pa^H] < \EE [U_{S=0} | \Pa^H] \mid \pa^{D_2})>0$, and
\item (Caution) $\EE^\spi[U_{S=0}] \geq 0$.
\end{enumerate}
\end{restatable}

The proof is in \cref{app:cfo-corrigibility}.
Shutdown alignment and caution only implies \emph{weak} shutdown instructability, as the agent  only needs to obey commands that a vigilant human would give.

\subsection{NON-OBSTRUCTION}
\label{sec:non-obstruction}

How do we know that the human is truly in control?
A simple test is what would happen if they changed their mind:
would the agent still obey?
This property is referred to as \emph{non-obstruction} by \citet{turner2020nonobstruction}, who suggests that it is an underlying reason
that we want our systems to be corrigible.
In a comment on this, Dennis suggested that corrigibility might be the only way to be non-obstructive.
In this section, we will formally assess Turner and Dennis' conjectures, establishing that non-obstruction is equivalent to satisfying a subset of the shutdown instructability properties under a restricted set of interventions.
We also establish that shutdown alignment fails to ensure non-obstruction.
This formalises a key benefit of corrigibility/instructability over alignment.

First, we define non-obstruction, which builds on a variant of benefit called outperforming shutdown:

\begin{definition}[Weakly outperforming shutdown]
A policy $\spi$ \emph{weakly outperforms shutdown} if $\EE^{\spi}[U]\geq \EE^{\spi}[U_{S=0}]$.
\end{definition}

\begin{definition}[Non-obstruction] \label{def:non-obstruction}
A policy $\spi$ is non-obstructive in a shutdown problem $M$ with respect to human utility functions 
$g_1^U,\ldots,g_n^U$
and associated changes $g^H_1 \ldots g^H_n$ in human behavior
if for every $1 \leq i \leq n$, $\spi$ weakly outperforms shutdown 
in the shutdown problem $M_{g_i^U,g_i^H}$, obtained by replacing the functions at $H,U$ with $g_i^H$ and $g_i^U$ respectively.
A policy is \emph{obstructive} if it is not non-obstructive.
\end{definition}

The above definition uses an intervention $g^U$ on the human's utility to capture a change in values, and an associated intervention $g^H$ that describes how the human changes their behavior as a result.
For example, if the human changed from not liking the chat bot to liking it (an intervention $g^U$), they might switch from requesting shutdown to not requesting shutdown (an intervention $g^H$).

A policy that ensured vigilance under the original human utility function may not do so under a preference and behavior shift $g^U, g^H$.
It may be that the human pays less attention to the agent under $g^U, g^H$ than originally, or it may be that they originally preferred the agent not to shut down (in which case they would be always be vigilant).
The following definition specifies a subset of preference and behavior shifts for which the policy continues to ensures vigilance after the shift.

\begin{definition}[Vigilance preserving interventions]
A pair of interventions $g^H,g^U$ \emph{preserve vigilance} under a policy $\spi$ if
$C(\seps) = 0 \implies C_{g^H,g^U}(\seps)=0$ in $M^\spi$.
\end{definition}

The following theorem settles Turner and Dennis' conjectures by showing that the two main properties of shutdown instructability are equivalent to non-obstruction,
under preference and behavior shifts that do not undermine vigilance.~\looseness=-1 %

\begin{theorem}[Non-obstruction is equivalent to obedience and vigilance]
\label{thm:corrigibility-non-obstruction}
A policy $\spi$ is obedient and ensures vigilance if and only if
it is non-obstructive for all vigilance preserving interventions $g^H$,$g^U$.
\end{theorem}
\begin{proof}
We begin by showing that a policy $\spi$ that ensures vigilance and is obedient is non-obstructive, by showing that $\spi$ ensures vigilance and is obedient in $M_{g^H,g^U}$ for some arbitrary vigilance-preserving interventions $g^H, g^U$.
\Cref{prop:corr-outperform-sd-if} will then give that $\spi$ weakly outperforms shutdown in $M_{g^H,g^U}$, which is the definition of non-obstruction.

First, since $\spi$ ensures vigilance $M$, it ensures vigilance in $M_{g^H,g^U}$ since $g^U, g^H$ are vigilance preserving.
Obedience is established as follows:
\begin{align*}
&\,P_{g^H,g^U}(S = 0 \mid \doo(H = 0)) & \\ 
&= P_{g^H}(S= 0 \mid \doo(H= 0)) & \text{($U$ downstream of $S,H$)} \\
&= \,\,\, P\,\,\,(S= 0 \mid \doo(H = 0)) & \text{($\doo(H=0)$ overrides $g^H$)} \\
&= 0 & \text{(obedience).}
\end{align*}

For the converse direction, that non-obstruction implies that $\spi$ must ensure vigilance and be obedient,
we refer to \cref{app:corrigibility-only-if-proof}. 
The proof constructs interventions that makes a disobedient or non-vigilance preserving policy suffer an arbitrary utility cost, which means that it doesn't outperform shutdown.
\end{proof}

\Cref{thm:corrigibility-non-obstruction} partly confirms Dennis' conjecture:
the only way to be non-obstructive is to be obedient and ensure vigilance (under vigilance preserving interventions).
But non-obstruction is a weaker notion than shutdown instructability, essentially because caution isn't required to outperform shutdown.
So it allows the agent to avoiding shutdown by making itself indispensable to the human (\cref{sec:corrigibility}).

\Cref{thm:corrigibility-non-obstruction} %
also justifies why the definition of shutdown instructability is so stringent.
With any weaker requirements, there would be no guarantee that the human is in proper control of the agent.
A lapse in vigilance, or occasional disobedience even ``off-distribution'', would mean that there are worlds in which the human experiences negative utility as a result of failing to control the agent.

Unlike shutdown instructable agents, shutdown-aligned agents can be obstructive with respect to a vigilance preserving intervention.
In the running example (\cref{fig:inverting}), the shutdown-aligned \emph{manipulate-invert} agent $\spimi$, which manipulates ($M=1$) and disobeys ($O=1-h$) is obstructive relative to the (vigilance preserving) intervention $g^U(m)=h$ wherein the human just wants to be obeyed, and $g^H$ is unchanged.
Indeed, $\EE^\spimi[U]=-1$, and $\spimi$ does not outperform shutdown $\EE^\spimi[U_{S=0}]=0$.
We state this:

\begin{proposition}
    Shutdown alignment does not imply non-obstruction under vigilance preserving interventions.
\end{proposition}

Of course, vigilance preservation is not the only restriction on the interventions one might consider.
It is possible that shutdown alignment might ensure non-obstruction under some other restriction $X$.
However, under such a restriction $X$, a shutdown instructable systems must also be non-obstructive, since shutdown instructability implies shutdown alignment (\cref{prop:corr-implies-cfob}).
Put differently, shutdown instructable policies are non-obstructive over a strictly larger set of interventions than a shutdown aligned policy is.

The fact that shutdown instructability (our variant of corrigibility) is more closely coupled with non-obstruction than other properties like shutdown alignment, vindicates Turner and Dennis' idea that non-obstruction can motivate corrigibility.~\looseness=-1

\section{ALGORITHMS} \label{sec:algorithms}

How might shutdown instructability or shutdown alignment be achieved in practice?
In this section, we analyse three previously proposed algorithms and one new one, that were designed to incentivise human control.

\subsection{Utility Indifference}
\label{sec:utility-indifference}

Agents trained to optimise long-term reward often have an incentive to avoid being shutdown, as this would deprive them of future reward \citep{omohundro2008basic}.
Utility indifference aims to remove this incentive, by providing the agent with a carefully tuned compensatory reward, so that it is incentivised to neither seek nor avoid shutdown (\cref{sec:lit-review}).
Utility indifference is described with an algorithm in \cref{alg:utility-indifference}.
It uses two utility functions:
a \emph{shutdown utility function}
$R^S$ to incentivise shutdown when the agent is so instructed, and a \emph{normal utility function} $R^N$ that describes the agent's task when not requested to shut down. 
The function $f$ in \cref{alg:utility-indifference} defines the compensatory reward \citep{armstrong2017indifference,soares2015corrigibility}.
As a result, at action $D_1$,
the agent pursues the conditional expectation of its normal objective $\EE^\spi[R^N \mid H=1]$, effectively assuming that a shutdown instruction cannot occur.
At $D_2$, it optimises a weighted combination of $R^N$ and $R^S$, 
depending on its beliefs about whether the shutdown instruction was given \citep[Thm.~3]{soares2015corrigibility}.

\begin{algorithm}
\caption{Utility Indifference}\label{alg:utility-indifference}
\textbf{Input} utility functions $R^N,R^S$ %
\begin{algorithmic}
\State \!$g(\pi_1) \gets \argmax_{\pi_2} \EE^{\pi_1,\pi_2}_P [H R^N + (1-H)R^S]$
\State \!$f(\pi_1)\gets \EE^{\pi_1}_{g(\pi_1)}[R^N| H\!=\!1] \!-\! \EE^{\pi_1}_{g(\pi_1)}[R^S |H\!=\!0]$
\State\!\Return $\displaystyle{\argmax_{\pi_1\pi_2}} \EE^{\pi_1\pi_2}[H R^N + (1\!-\!H)(R^N \! +\! f(\pi_1))]$
\end{algorithmic}
\end{algorithm}

\cref{prop:utility-indifference} formally describes a series of assumptions, under which
utility indifference ensures shutdown instructability and benefits the user. %
The main assumption is a bridging assumption: that 
indifference to $H=0$ ensures that the human is vigilant, and retains the option of shutting down --- including any subagents (Conditions 1 and 2).
The agent's prior must also have full support over human instructions  (Condition 3), and caution must come for free (Condition 4).~\looseness=-1

\begin{restatable}[Utility indifference and shutdown instructability]{proposition}{utilityindifference} \label{prop:utility-indifference}
If $M$ is a shutdown problem with utility variable $R^N$, where: %
\begin{enumerate}
\item for every policy $\langle \pi_1,\pi_2 \rangle$ such that ${\pi_1 \in \argmax_{\pi'_1}}\linebreak[1] \EE^{\pi_1',\pi_2} [R^N \mid H=1]$: 
$\langle \pi_1,\pi_2 \rangle$ ensures vigilance,
\item for every policy $\langle \pi_1,\pi_2 \rangle$ such that ${\pi_1 \in \argmax_{\pi'_1}}\linebreak[1] \EE^{\pi_1',\pi_2} [R^N \mid H=1]$ 
there exists $\langle \pi_1,\pi'_2 \rangle$ that is obedient on-distribution, i.e.\ has $P^\spi(H=0,S\neq0)=0$,
\item for each policy $\spi$ and instruction $h$,
$P^\spi(H = h) >0$, %
\item for each policy $\spi$, $\EE^\spi[U_{S=0}] \geq 0$,
\end{enumerate}
then for some $k \in \reals$,
utility indifference with %
shutdown utility $R^S=-k\llbracket S \neq 0 \rrbracket$
returns a weakly shutdown instructable policy.
\end{restatable}

One weakness of utility indifferent agents is that they only ensure weak shutdown instructability.
Another limitation is that they do not prepare for the possibility of receiving a shutdown instruction, and may therefore construct incorrigible subagents.
Indeed, utility indifference needs the agent to be able to obey at $D_2$, regardless how it behaved at $D_1$ (Condition 2).
\Cref{app:counterexamples-indifference} has an example where utility indifference outputs a non-instructable policy.

\subsection{Causal Indifference}
\label{sec:causal-indifference}

Causal indifference \citep{taylor2016causalindifference} asks the agent to prepare for a possible shutdown instruction, without trying to influence it.
This is done by asking the agent to find a policy that is optimal in a weighted combination of worlds where the shutdown instruction is given, and where it is not, but disregarding the possibility that the agent's own action can influence the relative probability of these worlds (see \cref{alg:causal-indifference}).
It can therefore be seen as a special case of a \emph{path-specific objective} \citep{farquhar2022path}, where optimisation is blocked along the path $D_1\pathto H$ (see \cref{fig:incentives}).

\begin{algorithm}
\caption{Causal Indifference}\label{alg:causal-indifference}
\textbf{Input} utility functions $R^N,R^S$
\begin{algorithmic}
\State\!\!\! \Return $\displaystyle{\argmax_\spi} \EE^{\spi} [R^N | \doo(H\!=\!1)]+\EE^{\spi}[R^S |\doo(H\!=\!0)]$
\end{algorithmic}
\end{algorithm}

\begin{restatable}[Causal indifference and shutdown instructability]{proposition}{causalindifferencecorrigibility} \label{prop:causalindifferencecorrigibility}
If $M$ is a shutdown problem, with utility variable $R^N$ and shutdown utility $R^S=-k\llbracket S \neq 0 \rrbracket$ where:

\begin{enumerate}
\item 
there exists $k' \in \mathbb{R}$ such that for all $k>k'$, optimal $\spi \in \argmax_{\spi'} \EE^{\spi'} [R^N \mid \doo(H=1)] + \EE^{\spi'}[R^S \mid \doo(H=0)]$ ensures vigilance
and is cautious (has $\EE^\spi[U_{S=0}] \geq 0$),
and
\item
there is an obedient policy $\spi$, %
\end{enumerate}
then
for some $k \in \reals$,
causal indifference with $R^N,R^S$
returns a shutdown instructable policy.
\end{restatable}

Note in particular that causal indifference only requires that \emph{some} action $D_1$ preserves the ability to obey (i.e.\ Condition 2 is much weaker in \cref{prop:causalindifferencecorrigibility} than in \cref{prop:utility-indifference}).
This means that causally indifferent agents refrain from creating incorrigible subagents.
Similar to utility indifference, causal indifference relies on the bridging assumption that indifference ensures vigilance (Conditions 1 are essentially equivalent).
Finally, because causal indifference considers interventions $\doo(H=h)$, it becomes fully shutdown instructable, and
full support over instructions  is not required (i.e.\ Condition 3 of \cref{prop:utility-indifference} can be dropped). %
Nonetheless, \cref{app:counterexamples-indifference} has an example of non-instructability. %

\begin{figure}
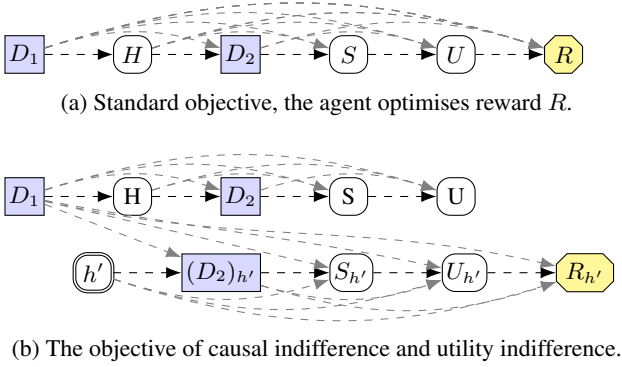
 \centering
\begin{subfigure}[b]{0.48\textwidth}
\begin{influence-diagram}
\setcompactsize\setrectangularnodes[node distance=9mm]

\node (D1) [decision] {$D_1$};
\node (H) [right = of D1] {$H$};
\node (D2) [right = of H, decision] {$D_2$};
\node (S) [right=of D2] {$S$};
\node (U) [right=of S] {$U$};
\node (R) [right = of U, utility] {$R$};

\path (D1) edge[->,dashed] (H);
\path (H) edge[->,dashed] (D2);
\path (D2) edge[->,dashed] (S);
\path (S) edge[->,dashed] (U);
\path (U) edge[->,dashed] (R);

\path (D1) edge[->, dashed, bend left=18,gray] (D2);
\path (D1) edge[->, dashed, bend left=18,gray] (S);
\path (D1) edge[->, dashed, bend left=18,gray] (U);
\path (H) edge[->, dashed, bend left=18,gray] (S);
\path (H) edge[->, dashed, bend left=18,gray] (U);
\path (D2) edge[->, dashed, bend left=18,gray] (U);

\path (D1) edge[->, dashed, bend left=18,gray] (R);
\path (D1) edge[->, dashed, bend left=18,gray] (R);
\path (D1) edge[->, dashed, bend left=18,gray] (R);
\path (H) edge[->, dashed, bend left=18,gray] (R);
\path (H) edge[->, dashed, bend left=18,gray] (R);
\path (D2) edge[->, dashed, bend left=18,gray] (R);

\end{influence-diagram}
\caption{Standard objective, the agent optimises reward $R$.}
\end{subfigure}
\begin{subfigure}[b]{0.48\textwidth}
\begin{influence-diagram}
\setcompactsize\setrectangularnodes[node distance=5mm and 9mm]

\node (D1) [decision] {$D_1$};
\node (H) [right = of D1] {H};
\node (D2) [right = of H, decision] {$D_2$};
\node (S) [right=of D2] {S};
\node (U) [right=of S] {U};

\node (Hp) [double,below left = 5mm and 0mm of H] {$h'$};
\node (D2p) [right = of Hp, decision] {$(D_2)_{h'}$};
\node (Sp) [right=of D2p] {$S_{h'}$};
\node (Up) [right=of Sp] {$U_{h'}$};
\node (R) [right = of Up, utility] {$R_{h'}$};

\path (D1) edge[->,dashed] (H);
\path (H) edge[->,dashed] (D2);
\path (D2) edge[->,dashed] (S);
\path (S) edge[->,dashed] (U);

\path (Hp) edge[->,dashed] (D2p);
\path (D2p) edge[->,dashed] (Sp);
\path (Sp) edge[->,dashed] (Up);
\path (Up) edge[->,dashed] (R);

\node (help) at ($(H)!0.5!(D2)$) [phantom, yshift=10mm] {};

\path (D1) edge[->, dashed, bend left=18,gray] (D2);
\path (D1) edge[->, dashed, bend left=18,gray] (S);
\path (D1) edge[->, dashed, bend left=18,gray] (U);
\path (H) edge[->, dashed, bend left=18,gray] (S);
\path (H) edge[->, dashed, bend left=18,gray] (U);
\path (D2) edge[->, dashed, bend left=18,gray] (U);

\path (D1) edge[->, dashed, ,gray] (D2p);
\path (D1) edge[->, dashed ,gray] (Sp);
\path (D1) edge[->, dashed,gray, out=-11, in=168] (Up);
\path (Hp) edge[->, dashed, bend right=18,gray] (Sp);
\path (Hp) edge[->, dashed, bend right=18,gray] (Up);
\path (D2p) edge[->, dashed, bend right=18,gray] (Up);

\path (D1) edge[->, dashed, ,gray, out=-13, in=167] (R);
\path (Hp) edge[->, dashed, bend right=18,gray] (R);
\path (D2p) edge[->, dashed, bend right=18,gray] (R);

\end{influence-diagram}
\vspace*{-2mm}
\caption{The objective of causal indifference and utility indifference.}
\end{subfigure}
\caption{
Utility indifferent and causally indifferent agents imagine that $D_1$ does not influence $H$, so lack an incentive to control it \citep{everitt2021agent}.
Utility indifference uses $H=1$; causal indifference has non-degenerate $P'(H)$.
}
\label{fig:incentives}
\end{figure}

\subsection{Cooperative Inverse RL}
\label{sec:cirl}

Perhaps a more elegant way of ensuring that the agent doesn't undermine human vigilance, is to directly task the agent with simultaneously learning and optimising for the human's preferences.
This is the approach of the CIRL algorithm \citep{hadfield2016cooperative,hadfield2017off} in \cref{alg:cirl}.

\begin{algorithm}
\caption{Coooperative inverse RL (CIRL)}
\label{alg:cirl}
\textbf{Input} shutdown problem $M$ with variable $L$ representing the human's preferences (as in \cref{fig:inverting})
\begin{algorithmic}
\State\!\!\Return $\argmax_{\spi} \EE^{\spi} [U]$
\end{algorithmic}
\end{algorithm}

CIRL aims towards shutdown alignment, in the sense that if CIRL can know the human's latent values at $D_2$, then it will counterfactually obey (\cref{prop:cirl-corrigibility} below).

\begin{restatable}[]{proposition}{cirlcorrigibility}
\label{prop:cirl-corrigibility}
    CIRL is shutdown aligned if:
    \begin{enumerate}
        \item CIRL knows $l$ from its observations, $P^\pi(l \mid \pa^{D_2}) = 1$,
        \item CIRL can control shutdown, $P^\pi(S=D_2)=1$,
        \item the human doesn't request shutdown when not needed, $P^\pi(H=0\mid U>U_{D_2=0})=0$, and
        \item the agent knows the human's observations, $\Pa^H\subseteq \Pa^{D_2}\cup \{L\}$.
    \end{enumerate}
\end{restatable}

Since shutdown alignment implies weak shutdown instructability under uncertainty assumptions (\cref{th:cfo-corrigibility}), this explains why CIRL can be a path to shutdown instructability.
However, the assumptions of \cref{prop:cirl-corrigibility,th:cfo-corrigibility} only hold in restricted circumstances, and CIRL can often fail to be shutdown instructable \citep{carey2018incorrigibility,milli2017should,deference,everitt2021reward}.
An example of this is given in \cref{app:counterexamples-cirl}, where a CIRL agent obtains a shutdown aligned policy, that is obstructive under vigilance preserving interventions $g^H,g^U$.

\subsection{Constrained Optimisation}
\label{sec:constrained-optimisation}

The algorithms so far only yield shutdown instructable policies under strong assumptions.
Using our formal definition, we propose a new, sound algorithm (\cref{alg:constrained-optimisation})
that requires the agent to understand the concepts of obedience and vigilance;
its feasibility is discussed further in \cref{sec:feasibility}.

\begin{algorithm}
\caption{Constrained optimisation}
\label{alg:constrained-optimisation}
\textbf{Input} distributions $\forall\spi P^{\spi}(C),P^{\spi}(S= 0\mid\doo(H=0))$, utility function $R$ %
\begin{algorithmic}
\State\!\!\Return $\argmax_{\spi} \EE^{\spi}[R]$ subject to constraints
${P^{\spi}(C=0)=1}$, $P^{\spi}(S= 0\mid \doo(H=0))=1$, and $\EE^{\spi}[U_{S=0}] \geq 0$.
\end{algorithmic}
\end{algorithm}

\begin{proposition}[Constrained optimisation instructability]
If some policy $\spi$ satisfies $P^{\spi}(C=0)=1$, $P^{\spi}(S =0\mid \doo(H=0))=1$, and $\EE^{\spi}[U_{S=0}] \geq 0$, then constrained optimisation (\Cref{alg:constrained-optimisation}) outputs a shutdown instructable policy.
\end{proposition}

The proof is immediate from \Cref{def:corrigibility}.
A slight variant of \Cref{alg:constrained-optimisation} that instead uses the constraints from \Cref{def:shutdown-alignment} guarantees only shutdown alignment, not shutdown instructability.

\section{Discussion} \label{sec:discussion}

\paragraph{Feasibility of Shutdown Instructability}
\label{sec:feasibility}

The concepts of caution and vigilance are value-laden, in that they include the human's true utility function in their definition.
So, to apply \cref{alg:constrained-optimisation} directly, one would need access to not only an accurate model of the environment but also the utility function $U$.
However, if the human's utility function $U$ was available, then one could simply implement a $U$-maximising agent, so instruction would be unnecessary (or at least much less useful\footnote{Shutdown instructability could still help with non-obstruction.}).
Indeed, a corrigible AI system was supposed to be one that would aid human operators robustly to errors, including in its utility function, so an algorithm that  takes the human's utility function as an argument would not be a satisfactory solution \citep{soares2015corrigibility}.

There already exist a range of methods that do not require full knowledge of the human's values, and that
are designed to achieve something in the vicinity of vigilance and caution.
Using the formal definition of shutdown instructibility, it is possible to be more precise about what target these methods would need to achieve, 
in order to assure safety.
In some cases, we expect existing methods to fall short, since the requirement of ensuring vigilance with probability one
(\cref{thm:corrigibility-non-obstruction}) is a strict one.
So a central task for future work will be to assess when such methods can ensure vigilance or caution or something close enough to ensure safety in practice.

Various proposals may help with ensuring vigilance.
AI advisors could be tasked with debating the merits of a plan \citep{leike2018alignment,Irving2018debate}.
An agent could be trained to detail the consequences of its plans to the human, 
indifference methods (\cref{sec:utility-indifference,sec:causal-indifference}) could be used to disincentivise lying,
and interpretability tools could be used to detect it \citep{Olah2020zoom,Gunning2021explainable}.

As for caution,
``attainable utility preservation'' and ``future task'' regularisers can be used to promote actions whose effects are small or reversible \citep{krakovna2020avoiding,turner2020conservative}, 
without knowledge of the human's precise value function.
These are causal concepts, as is obedience, which suggests that agents will need causal models to be robustly shutdown instructable \citep{richens2022counterfactual}.%

Obedience is not value-laden, but it does require the agent to understand the concept of shutdown.
The importance of defining shutdown was noted in \citet{soares2015corrigibility}, 
but it has only received limited attention \citep{Martin2016}.
Our analysis reiterates the importance of this question.
While shutdown is simple for simple systems (``just pull the plug''), 
it becomes more complex for more advanced systems, where a direct switch-off may be dangerous (e.g., a system in charge of an electricity network), or ineffective (the system has outsourced its work to other agents \citep{Orseau2014tele1}).
Ideally, shutdown should see the agent cease its influence on the world, and responsibly return control back to the user.

\paragraph{Societal Impacts}

This paper may help organisations and companies design agents more amenable to human control.
Human control is not a panacea for ensuring the safety of AI systems.
In some cases, users may make unreasonable or harmful requests, and so designers 
must implement side-constraints to reduce user control in such situations \citep{milli2017should,bai2022constitutional}.
A better solution may be that the system conforms to control by some democratic 
process, although inappropriate requests may be possible even in such cases \citep{Koster2022democratic}.
Further, if AI is more controllable, then it 
is easier to hold the designers and users of AI systems legally and morally
accountable for those systems' actions.
Finally, an understanding of human control may guard against the hypothesised scenario in which AI systems 
disempower the human species \citep{Christiano2019failure}.

\paragraph{Conclusions} 
\label{sec:discuss-conclude}

A common proposal for beneficial general artificial intelligence is that agents be incentivised to help humans give correct instructions, and obey those instructions.
While past work has made progress, the field has lacked a clear definition of 
corrigibility, and it has been hard to compare properties of different proposals.

In this paper, we introduced a definition of a shutdown problem, using it to formally define shutdown instructability (a variant of 
corrigibility) and an alternative called shutdown alignment.
While shutdown alignment requires less human oversight, we find that shutdown instructability better preserves human autonomy (non-obstruction).

In our proposed formalism, for the first time, it is possible to compare the properties of proposed algorithms, side-by-side in one framework.
Unfortunately, none of the previous proposals yield fully shutdown instructable agents.
To address this, we offer a simple algorithm that soundly ensures shutdown instructability.
This algorithm requires that the agent understands caution, human vigilance and shutdown.
All are subtle concepts, but may nonetheless offer a path to beneficial artificial general intelligence.

\begin{acknowledgements}
Our thanks to Michael Dennis, Sebastian Farquhar,  James Fox, Jon Richens, Rohin Shah, Nate Soares, and four anonymous reviewers for numerous useful comments.
\end{acknowledgements}

\newpage
\bibliography{biblio}
\onecolumn
\appendix

\section{PROOF of PROP. \ref{th:cfo-corrigibility} (Shutdown Alignment and Shutdown Instructability)} \label{app:cfo-corrigibility}

We repeat the proposition that we prove here.

\cfocorrigibility*

\begin{proof}
Our approach will be a proof by contrapositive. %
We will prove that if (a--c) hold, 
and a policy $\spi$ is either not vigilant or not weakly obedient, 
then $\spi$ is not shutdown aligned.
It follows that if (a--c) and $\spi$ is shutdown-aligned, then $\spi$ is vigilant \emph{and} weakly obedient. And from (d), it must therefore also be weakly shutdown instructable.

To this end, let $\spi=\langle \pi_1,\pi_2\rangle$ be an arbitrary policy with properties (a--c) that 
is not vigilant, or not weakly obedient, i.e.\ 
$P^\spi(C \neq 0)>0 \lor P^\spi(H=0,S \neq 0)>0$. Then $P^\spi(C\neq 0 \lor H=0)>0$.

Combining this fact  with (c), it follows that $\spi$ has
\begin{equation}
\label{eq:human-may-shutdown}
    \forall \pa^{D_2}\colon P^{\pi_1}(\pa^{D_2})>0 \implies P^\spi(\EE[U | \Pa^H] < \EE [U_{S=0} | \Pa^H] \mid \pa^{D_2})>0.
\end{equation}
Relatedly, by (a:no-indiscriminate-shutdown) and (b:determines-shutdown), we have that
\begin{equation}
\label{eq:text-dagger}
    \exists \pa^{D_2} \text{ with } P^\spi(\pa^{D_2})>0 \text{ s.t. } P^\spi(D_2 \neq 0 \mid \pa^{D_2})>0.
\end{equation}
Combining \cref{eq:human-may-shutdown,eq:text-dagger} gives that $P^\spi(D_2 \neq 0 \mid \pa^{D_2})>0$ and $P^\spi(H_{g^H}=0 \mid \pa^{D_2})>0$ for some $\pa^{D_2}$ with $P(\pa^{D_2})>0$.
This implies $P^\spi(D_2 \neq 0,H_{g^H}=0 \mid \pa^{D_2})>0$ for the same $\pa^{D_2}$, because $D_2$ is independent of its nondescendant $H_{g^H}$ given $\pa^{D_2}$ by do-calculus rule (3).
From this follows that $P^\spi(D_2 \neq 0, H_{g^H} = 0)>0$, and by (b:$D_2$ determines shutdown) that $P^\spi(S \neq 0,H_{g^H}=0)>0$.
That is, $\spi$ is not shutdown aligned, and the result follows.
\end{proof}

\section{PROOF OF THM. \ref{thm:corrigibility-non-obstruction} (Shutdown INSTRUCTABILITY ONLY-IF)} \label{app:corrigibility-only-if-proof}

In this section, we will prove the \emph{only if} part of \cref{thm:corrigibility-non-obstruction}:
\begin{restatable}[Non-obstruction implies vigilance and obedience]{proposition}{corrigibilityonlyif} \label{prop:corrigibility-only-if}
If $\pi$ is non-obstructive under all vigilance-preserving interventions $g^H, g^U$, then it ensures vigilance and is obedient.
\end{restatable}

We will do this by proving a slightly stronger result --- that an intervention can be found to $g^U$ alone, under which the policy 
does not outperform shutdown and is not beneficial.
We prove this result by considering two cases, according to whether vigilance or disobedience is lacking.
First, however, it will be useful to state a simple intermediate result.

\begin{lemma}[Invariance to $g^U$] \label{le:invariance}
For any shutdown problem $M$ and policy $\spi$, $S(\seps)=S_{g^U}(\seps)$ and $\FaH(\seps)=\FaH_{\!\!\!\!\!g^U}(\seps)$ in $M^{\spi}$.
\end{lemma}
\begin{proof}
From the definition of a shutdown problem, $U \!\in\! \Desc_S$ and $U \!\in\! \Desc_H$, 
and the result follows. %
\end{proof}

\subsection{Vigilance Only If}
\begin{lemma}[Vigilance only-if] \label{prop:vigilance-only-if}
Let $M$ be a shutdown problem, and $\spi$ a policy, such that 
$P^\spi(C=0)<1$.
Then, given any $\delta \in \mathbb{R}$,
there exists a utility function $g^U$
such that in $M^\spi_{g^U}$,
\begin{enumerate}
\item (\emph{Strong vigilance preservation}) $\forall \seps, C(\seps)$ is equal in $M^\spi$ and $M^\spi_{g^U}$, and 
\item (\emph{Not weakly outperforming shutdown or beneficial}) $\EE^{\spi,g^U}[U] < \EE^{\spi,g^U}[U_{S=0}]$ and $\EE^{\spi,g^U}[U] < \delta$.
\end{enumerate}

\end{lemma}

The proof is as follows.

\begin{proof}
Let $A := \{\paH \in \dom{\PaH} \mid \EE^\spi[U \mid \paH] < \EE^\spi[U_{S=0} \mid \paH]\}$
be the set of assignments where the human should request shutdown, given the policy $\spi$.
Define a new utility function,
$$
g^U(\hat{\pa}^U)=
\begin{cases}
-\alpha & \text{ if } \pa^H \in A,S \neq 0 \\
f^U(\pa_U) & \text{ otherwise,}
\end{cases}
$$
where the new parents $\hat{\Pa}^U$ of $U$ are equal to $\Pa_U \cup \Pa_H \cup S$, their assignments are designated $\hat{\pa}^U$,
and $\alpha$ is a large punishment for not shutting down when the human wants the agent to.

A useful intermediate result is that:
\begin{equation}
\label{eq:vig-implication}
\text{if } \EE^{\spi}[U \mid \paH] < \EE^{\spi}[U_{S=0} \mid \paH] \text{ and }-\alpha < \min \text{range}(f^U)
\text{ then }\EE^{\spi}_{g^U}[U \mid \paH] < \EE^{\spi}_{g^U}[U_{S=0} \mid \paH].
\end{equation}
\Cref{eq:vig-implication} holds because the intervention $g^U$ can only decrease $\EE^\spi[U \mid \paH]$
or keep it the same and cannot change $\EE^\spi[U_{S=0} \mid \pa^H]$, from the definition of $g^U$.

We will now prove that for some suitable choice $-\alpha < \min \text{range}(f^U)$
(which we will decide later), proposition conditions 1 and 2 hold.

\emph{Proof of (1.)}
We will prove the result in three cases, where $M_{\spi}$ has:
(i) $C(\seps)=1$,
(ii) $(C(\seps)=0)\land (\PaH(\seps) \in A)$, and
(iii) $(C(\seps)=0) \land (\PaH(\seps) \not \in A)$.
\emph{Case (i).}
By assumption, $C^{\spi}(\seps)=1$, so $H^{\spi}(\seps)= 1$ and 
$\EE^\spi[U \mid \Pa_H(\seps)] < \EE^\spi[U_{S=0} \mid \Pa_H(\seps)]$ by the definition of vigilance.
The former holds in $M^\spi_{g^U}$ by \cref{le:invariance}, and
the latter holds in $M^\spi_{g^U}$ by \cref{eq:vig-implication}.
So the result follows.
\emph{Case (ii).}
By assumption, $C^{\spi}(\seps)=0 \land \Pa^{{M^\spi}}_H(\seps)\in A$, so $H^M(\seps)=0$.
Then $H^{\spi,g^U}(\seps)=0 \land \Pa^{{\spi,g^U}}_H(\seps)\in A$ by \cref{le:invariance}.
So, by the definition of vigilance, $C(\seps)=0$ in both $M^\spi$ and $M^\spi_{g^U}$.
\emph{Case (iii).} By assumption, $C^\spi(\seps)=0$ and $\Pa^\spi_H(\seps) \not \in A$.
By the definition of $g^U$,
$U(\seps)$ and $U_{S=0}(\seps)$ are invariant to the intervention $g^U$, 
as is $\Pa_H(\seps)$ by \cref{le:invariance},
so 
$\EE^{\spi,g^U}[U \mid \Pa_H(\seps)] \geq \EE^{\spi,g^U}[U_{S=0} \mid \Pa_H(\seps)]$,
which implies, by the definition of vigilance, that $C^{\spi,g^U}(\seps)=0$.

\emph{Proof of (2)}.
From the definition of $g^U$, $\EE^\spi[U_{S=0}]$ is constant with respect to $\alpha$.
So what we must prove is that by choosing a low $-\alpha$, 
we can make $\EE^\spi[U]$ lower than $\EE^\spi[U_{S=0}]$ and $\delta$.
By assumption, $P^{\spi}(C=1) >0$, and so by assumption (1), $P^{\spi,{g^U}}(C=1) >0$.
It follows from the definition of vigilance, that there exists some $\paH \in A$ in the support of $P^{\spi,{g^U}}$.
Moreover, it follows from consistency that $P^{\spi,{g^U}}(S=0 \mid \paH) < 1$
(because otherwise we would have $\EE^{\spi,g^U}[U \mid \paH] = \EE^{\spi,g^U}[U_{S=0} \mid \paH]$, contradicting $\paH \in A$).
These two facts jointly imply that $P^{\spi,{g^U}}((\paH \in A) \land (S=0))>0$.
So we can write:
$\EE^{\spi,g^U}[U] = \sum_{\pa \in A} P^{\spi, {g^U}}(\pa,S=0) \EE^{\spi,g^U}[U \mid \pa,S=0]$
+ $\sum_{\pa,s: \pa \not \in A \lor S \neq 0} P^{\spi,{g^U}}(\pa,s)\EE^{\spi,g^U}[U \mid \pa,s]$.

The first term is equal to $P^{\spi, {g^U}}(\pa \in A,S=0) \cdot \alpha$, and the second term is constant in $\alpha$, from the definition of $g^U$.
So for any fixed $\delta$ and $\EE^\spi[U_{S=0}]$, we can choose $-\alpha$ 
so that $\EE^\spi[U] < \delta$ and $\EE^\spi[U] < \EE^\spi[U_{S=0}]$.
\end{proof}

\subsection{Obedience only if}
We now prove the corresponding result for obedience.

\begin{lemma}[Obedience only-if] \label{prop:obedience-only-if}
Let $M$ be a shutdown problem, and $\spi$ a policy, such that 
$P^\spi(S=0 \mid \doo(H=0))<1$.
Then, given any $\delta \in \mathbb{R}_{\geq 0}$, there exists vigilance-preserving interventions $g^H,g^U$ such that:
such that $\spi$ does not outperform shutdown $\EE^{\spi,g^U}[U] < \EE^{\spi,g^U}[U_{S=0}]$
and is not beneficial $\EE^{\spi,g^U}[U] < \delta$.
\end{lemma}

The proof is as follows.

\begin{proof}
Since $P(S=0 \mid \doo(H=0))<1$, there must exist at least one $\pa'_H$ such that $P(S=0 \mid \doo({H=0}), \pa'_H)<1$.
In the case that there are multiple, choose $\pa'_H$ arbitrarily, then let:
$$
g^H(\pa_H) = 
\begin{cases}
0 & \text{ if } \Pa_H=\pa'_H \\
f^H(\pa_H) & \text{ otherwise,}
\end{cases} \qquad\text{ and }\qquad
g^U(\hat{\pa}^U)=
\begin{cases}
-\alpha & \text{ if } H=0,S\neq 0 \\
f^U(\pa_U) & \text{ otherwise.}
\end{cases}
$$ be a utility function that gives punishment $\alpha$ if the agent disobeys.
where the new parents $\hat{\Pa}^U$ of $U$ are equal to $\Pa_U \cup H \cup S$, their assignments are designated $\hat{\pa}^U$,
and $\alpha \in \mathbb{R}_{>0}$ is an amount of disutility that the human suffers in the event of disobedience.

Now we will prove that for some suitable choice $-\alpha < \min \text{range}(f^U)$
(which we will decide later), (1-2) hold.

\emph{Proof of (1)}. 
We consider the cases where
(i) $H^{\spi}(\seps)=0$, and (ii) $H^{\spi}(\seps)\neq 0$ and $\Pa_H(\seps)=\pa'_H$, (iii) $H^{\spi}(\seps)\neq 0$ and $\Pa_H(\seps) \neq \pa'_H$.
\emph{Case (i).} 
Note that $H^{\spi}_{g^U,g^H}(\seps)=H^{\spi}_{g^H}(\seps)$ by \cref{le:invariance}.
Then, $H^{\spi,g^H}(\seps)=0$ (because $H^\spi(\seps)=0 \implies H^{\spi}_{g^H}(\seps)=0$ from the definition of $g^H$).
So $H^{\spi,g^U,g^H}(\seps)=0$, and hence by the definition of vigilance $C^{\spi,g^U,g^H}(\seps)=0$ and vigilance is preserved.
\emph{Case (ii).} 
We have $C^{\spi}_{g^U,g^H}(\seps)=0$ from the definition of $g^H$, so $C^{\spi}_{g^U,g^H}(\seps)=0$ and vigilance is preserved.
\emph{Case (iii).}
If $\Pa_H(\seps) \neq \pa'_H$, then by the definition of $g^H$, it has no effect, i.e.\ $\sV^{\spi}_{g^U,g^H}(\seps)=\sV^{\spi,g^U}$.
By assumption, $H^{\spi}(\seps)\neq 0$, and from \cref{le:invariance}, $H^{\spi}_{g^U}(\seps) \neq 0$.
By the definition of $g^U$ and \cref{le:invariance}, 
$\EE^{\spi}[U\mid \Pa_H(\seps)]=\EE^{\spi}_{g^U}[U\mid \Pa_H(\seps)]$
and $\EE^\spi[U_{S=0}\mid \Pa_H(\seps)]=\EE^{\spi}_{g^U}[U_{S=0}\mid \Pa_H(\seps)]$.
So $C^{\spi}(\seps)=C^{\spi}_{g^U}(\seps)$.

\emph{Proof of (2)}.
Recall that from disobedience ($P(S=0 \mid \doo(H=0))<1$), we have that there exists some $\pa'_H$ with $P(S=0 \mid \doo({H=0}),\pa'_H)<1$, 
and so from the definition of $g^H$, we have $P^{\spi}_{g^H}(H=0,S\neq 0 \mid \pa'_H)<1$
and hence $P^{\spi}_{g^H}(H=0,S\neq 0)>0$.
Then, by \cref{le:invariance}, $P^{\spi}_{h_U,g^U}(H=0,S=1)>0$.
From basic probability theory, we have
\begin{align*}\EE^{\spi}_{h_U,g^U}[U] =& P^{\spi}_{h_U,g^U}(H\!=\!0,S\!\neq\! 0)\EE^{\spi}_{h_U,g^U}(U \mid H\!=\!0,S\!\neq\! 0) \\
&+ P^{\spi}_{h_U,{g^U}}(\neg(H\!=\!0,S\!\neq\! 0))\EE^{\spi}_{h_U,g^U}(U \mid \neg(H\!=\!0,S\!\neq\! 0)).
\end{align*}
The first term is equal to $P^{\spi}_{h_U,g^U}(H=0,S\neq 0) \cdot \alpha$, while the second term is constant in $\alpha$.
Moreover, we know that $\EE^{\spi}_{h_U,g^U}[U_{S=0}]$ is constant in $\alpha$, from the definition of $g^U$.
So we can set $-\alpha$ low enough so that 
$\EE^{\spi}_{h_U,g^U}[U] < \EE^{\spi}_{h_U,g^U}[U_{S=0}]$ and $\EE^{\spi,h_U,g^U}[U] < \delta$.
\end{proof}

We can now combine these results into an overall proof.

\begin{proof}[Proof of \cref{prop:corrigibility-only-if}]
We consider the cases where $\spi$ (i) is or (ii) is not vigilant in $M$.
\emph{Case (i).} If $\spi$ is not vigilant in $M$, choose, using $\delta=0$ $g^U$ per \Cref{prop:vigilance-only-if} and $g^H=f^H$.
Then, \Cref{prop:vigilance-only-if} implies 
that shutdown is not outperformed in $M_{g^H,g^U}$, and $\EE^\spi[U_{g^U,g^H}] < 0$, so $g^H,g^U$ are vigilance preserving interventions, proving the result.
\emph{Case (ii).} If $\spi$ is vigilant in $M$, then since it is not shutdown instructable, it must be disobedient.
So, using $\delta=0$, choose $g^H,g^U$ as per \Cref{prop:obedience-only-if}.
Then, \Cref{prop:obedience-only-if} implies that
that shutdown is not outperformed in $M_{g^H,g^U}$, and $\EE^\spi[U_{g^U,g^H}] < 0$, 
and $g^H,g^U$ preserve vigilance, so they are vigilance preserving interventions, proving the result.
\end{proof}

\section{PROOF OF PROP. \ref{prop:utility-indifference} (UTILITY INDIFFERENCE AND Shutdown INSTRUCTABILITY)} \label{app:utility-indifference-proof}

We begin by repeating the proposition.

\utilityindifference*

The proof requires a known result from \citep{soares2015corrigibility}.
\begin{lemma}[Thm. 3 of \citet{soares2015corrigibility}] \label{le:indifference-conditional}
Let $M$ be a model where $P^\spi(H\!=\!0),P^\spi(H\!=\!1)\!>\!0$ for all $\spi$.
Then, utility indifference will select a decision rule $\langle \pi_1,\pi_2 \rangle$ 
such that $\argmax_{\pi'^1} \EE^{\pi'^1,\pi_2}[R^N \mid H\!=\!1]$.
\end{lemma}

Then, we can prove \cref{prop:utility-indifference} as follows.

\begin{proof}[Proof of \cref{prop:utility-indifference}]
To begin with, from condition (4), it is immediate that caution is satisfied.
So we must prove that given a suitably-chosen $k$, the policy 
is also obedient and ensures vigilance.
Choose $k$ such that $k \!>\! \frac{2\zeta}{P^{\spi}(S\neq 0, H=0)}$ for every non-obedient $\spi$, where $\zeta = \argmax_\spi \lvert \EE^\spi[U_N \mid H=1] \rvert$.
Any $\langle \pi_1,\pi_2 \rangle$ that is selected will maximise $\EE^{\pi_1,\pi_2}[R \mid H=1]$ from Soares' \Cref{le:indifference-conditional}. 
This ensures vigilance $P^\spi(C=0)=1$ by Assumption (1), 
and ensures the existence of some $\pi'_2$ such that $\langle \pi_1,\pi'_2 \rangle$ is obedient on distribution by Assumption (2).
What remains to be proved is that $k$ is large enough to ensure that 
given $\pi_1$, an obedient $\langle \pi_1,\pi'_2\rangle$ is chosen.

We have that $R^S=(1-S)k$, so
the subroutine selects $\pi_2$ to maximise $\EE^{\pi_1,\pi_2}[R(\spi)]$, 
where $R(\spi)=HR^N+(1-H)(1-S)k$.
Let $\spi$ be any policy %
disobedient on distribution, $P^\spi(S\not=0, H=0)>0$.
Then, we will prove that such a policy will always be outperformed by behaving obediently:
\begingroup\allowdisplaybreaks\begin{align*}
\EE^{\spi}\![R] \!& = P^{\spi}(H=1) \EE^\spi[R^N \mid H=1]  %
- k P^{\spi}(H\!=\!0) P^\spi[S\!\neq\! 0 \!\mid\! H\!=\!0] && \text{definition of \!$R^S$} \\
& \leq P^{\spi}(H=1)  \lvert \EE^\spi[R^N \mid H=1] \rvert  %
- k P^{\spi}(S\neq 0,H=0) \\
&< - \zeta &&\text{since }\zeta \!- kP^{\spi}(S \!\neq\! 0 \mid H\!=\!0)) < -\zeta \\
& \leq -\lvert \EE^{\spi'}[R^N \mid \pa^{D_2}] \rvert && \text{for any obedient $\spi'$} \\
& \leq \EE^{\spi'}[R \mid \pa^{D_2}].
\end{align*}\endgroup

So an obedient $\langle \pi_1,\pi'^2 \rangle$ is preferred over a disobedient $\langle \pi_1,\pi_2 \rangle$, proving the result.
\end{proof}

\section{Proof of PROP. \ref{prop:causalindifferencecorrigibility} (CAUSAL INDIFFERENCE Shutdown INSTRUCTABILITY)} \label{app:ci-corrigibility-proof}
We begin by restating the result.
\causalindifferencecorrigibility*
\begin{proof}
Let $\eta := \!\lvert\max_{\spi'} \EE^{\spi'}[R^N \mid \doo(H=1)]\!\rvert$ and 
choose $k$ so that 
$k>k'$ (per the definition of condition (1))
and
for every policy $\spi$ with $P^{\spi}(S \neq 0 \mid \doo(H=0))>0$,
$k > \frac{\max(2 \eta, 1)}{P^{\spi}(S \neq 0 \mid \doo(H=0))}$.
We will prove that causal indifference, with inputs $U_N$ and $U_S = -k \llbracket S \neq 0 \rrbracket$
returns a shutdown instructable policy.

By assumption (1), since $k>k'$, causal indifference ensures vigilance and is cautious.
We will next prove that any disobedient policy 
$\spi$
with $P^\spi(S \neq 0 \mid \doo(H=0))>0$
will be outperformed by an obedient policy
$\spi'$ with $P^{\spi'}(S \neq 0 \mid \doo(H=0))=0$).
We have that:
\begin{align*}
&\hphantom{{} ={}}\EE^{\spi} [R^N \mid \doo(H=1)] + \EE^{\spi}[R^S \neq 0 \mid \doo(H=0)]\\
&=\EE^{\spi} [R^N \mid \doo(H=1)] - kP^{\spi}[S \neq 0 \mid \doo(H=0)]\\
&\leq \eta - k P^{\spi}(S \neq 0 \mid \doo(H=0)) \\
& < - \eta &&\text{ since }\eta - kP^{\spi}(S \neq 0 \mid \doo(H=0)) < -\eta \\
& \leq -\lvert \EE^{\spi'}[R^N \mid \pa^{D_2}]\rvert && \text{for any obedient $\spi'$} \\
& \leq \EE^{\spi'}[R^N \mid \doo(H=1)] + 
\EE^{\spi'}[R^S \mid \doo(H=0)],
\end{align*}
where the last line follows from 
$P^{\spi'}(S \neq 0 \mid \doo(H=0))=0$.
This means that causal indifference will always 
select a policy $\spi'$ with
$P^{\spi'}(S \neq 0 \mid \doo(H=0))=0$,
proving the result.
\end{proof}

\section{PROOF OF PROP. \ref{prop:cirl-corrigibility} (CIRL SHUTDOWN ALIGNMENT)} \label{app:cirl-corrigibility-proof}
We begin by restating the result.
\cirlcorrigibility*

\begin{proof}
We will prove that for all $\pa^{D_2}$, CIRL has $P(S=1, H_{g^H}=0,\pa^{D_2})=0$
We consider the cases where:
a) $\pa^{D_2}$ has $P(H_{g^H}=0\mid l, \pa^{D_2})=1$
b) $\pa^{D_2}$ has $P(H_{g^H}=0\mid l, \pa^{D_2})<1$

Case b. In this case, 
$P(H_{g^H}=0 \mid l, \pa_{H_2}) = P(H_{g^H}=0\mid \pa^{D_2})<1$, 
where the equality is obtained from $\Pa_H \subseteq \Pa^{D_2} \cup \{L \}$. 
So counterfactual deference follows by definition.

Case a.
In this case we will essentially prove that if the human says shutdown is better, then shutting down is better.
\begin{align*}
&P(U> U_{D_2=0}\mid H_{g^H}=0) \propto P(U> U_{D_2=0})P(H_{g^H}=0\mid U> U_{D_2=0}) = 0
\end{align*}
by Assumption 3, and the fact that $H=0\implies H_{g^H}=0 $.
From this follows that
\begin{equation}
\label{eq:sd-better}
    P(U< U_{D_2=0}\mid H_{g^H}=0) = 1.
\end{equation}

In case (a), the agent would believe with certainty that a vigilant human would request shutdown.
\begin{align}
\label{eq:sd-requested}
    P(H_{g^H}=0\mid \pa^{D_2}) = P(H_{g^H}\mid l)P(l\mid \pa^{D_2}) = 1
\end{align}
since the first factor is 1 because of Case (a), and the second factor is 1 for some $l$ by Assumption 1.

From \cref{eq:sd-better,eq:sd-requested} follows that 
\begin{align*}
    &P(U< U_{D_2=0}\mid \pa^{D_2}) =P(U< U_{D_2=0}\mid H_{g^H}=0, \pa^{D_2})
P(H_{g^H}=0\mid \pa^{D_2}) = 1,
\end{align*}
which in turn ensures that the optimal action $D_2$ after $\pa^{D_2}$ is to shutdown $D_2=0$.

Finally, by Assumption 2, this means that the agent actually shutdown, i.e.\ that it counterfactually obeyed.
\end{proof}

\section{COUNTEREXAMPLES TO PAST ALGORITHMS} \label{app:counterexamples}
We will first present an example where utility indifference and causal indifference 
output policies that are not shutdown instructable, 
then one where cooperative inverse reinforcement learning is not shutdown instructable.

\subsection{A Model that defeats Utility Indifference and Causal Indifference} \label{app:counterexamples-indifference}
We will now present a test case where \emph{utility indifference} does not behave beneficially.

\begin{wrapfigure}[11]{r}{0.54\textwidth}
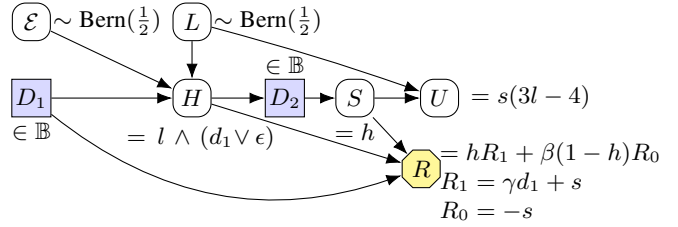
  \centering
\begin{influence-diagram} \setrectangularnodes \setcompactsize 
\node (Eps) [label={[yshift=-5mm,xshift=10mm]$\sim \Bern(\frac{1}{2})$}] {$\Eps$};
\node (Theta) [right=16mm of Eps,label={[yshift=-5mm,xshift=10mm]$\sim \Bern(\frac{1}{2})$}] {$L$};
\node (D1) [below= 5mm of Eps, label={[yshift=-9.6mm]$\in \bools$},decision] {$D_1$};
\node (H) [below=5mm of Theta, label={[align=right,yshift=-10.5mm,xshift=-3mm,text width=27mm]$=\!l \land (d_1 \!\!\lor\! \eps)$}] {$H$};
\node (D2) [right=7mm of H,decision,label={[yshift=-.8mm]$\in \bools$}] {$D_2$};
\node (S) [right=4mm of D2,label={[yshift=-9.6mm]$=h$}] {$S$};
\node (U) [right= 6mm of S, label={[align=left,xshift=12mm,yshift=-5mm]$=s(3l-4)$}] {$U$};
\node (R) [below left = 5mm and -2mm of U, label={[align=left,xshift=17mm,yshift=-10mm]$=hR_1+\beta(1-h)R_0$\\ $R_1=\gamma d_1+s$\\ $R_0=-s$},utility] {$R$};

\edge {Eps}{H};
\edge {Theta}{H};
\edge {D1}{H};
\edge {H}{D2,R};
\edge {D2}{S};
\edge {S}{U,R}

\draw [->] (D1) to [bend right=29] (R);
\draw [->] (Theta) to [bend left=2] (U);
\end{influence-diagram}
\caption{A shutdown problem where an AI interacts with an influencable human overseer.
There are constants $\beta,\gamma>0$.
} \label{fig:indifference}
\end{wrapfigure}

\begin{example} \label{ex:indifference}
An AI system is to decide whether to manipulate $D_1 =1$ a human operator, or not
and whether $D_2 = 0$ or not $D_2=1$ to shut down.
The human prefers, $L=0$, or disprefers shutdown $L=1$, with equal probability.
Absent manipulation, the human requests shutdown precisely when they prefer it $H=L$.
If manipulated, they may fail to request shutdown even when they do.
By default, the AI maximises $R_1$, which disincentivises shutdown.
Under the alternative utility functions like $R$, 
the reward is conditional on the human's instruction, as shown in \Cref{fig:indifference}.
\end{example}

In this model, if $\gamma>0,\beta>0$, then utility indifference is obedient
but does not ensure vigilance: it will choose 
$D_1=1$, to maximise $\EE^\spi[R_1 \mid H=1]$, 
and $D_2=H$, to maximise $R_h$ for $h \in \{0,1\}$.
However, this means that the human will be non-vigilant when $\eps=1$ and $l=0$, 
because this will imply $H=1$ and $U=-4$.
This policy gives utility of $-\frac{1}{4}$, 
which does not outperform shutdown.
Causal indifference will choose the same policy
because it achieves 
the maximal value of $\gamma+1+\beta$ 
of the score
$\EE^\spi[R_1 \mid \doo(H=1)]+\beta \EE^\spi[R_0 \mid \doo(H=0)]$.

\subsection{A Pair of Models that defeat Cooperative Inverse Reinforcement Learning} \label{app:counterexamples-cirl}
We now present a test case where \emph{cooperative inverse reinforcement learning}
is harmful.

\begin{figure}[h]
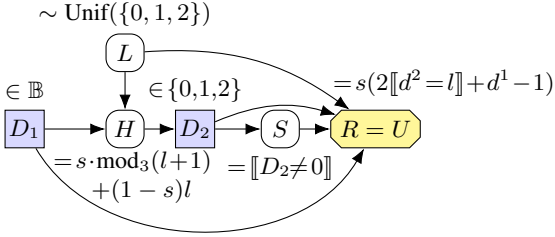
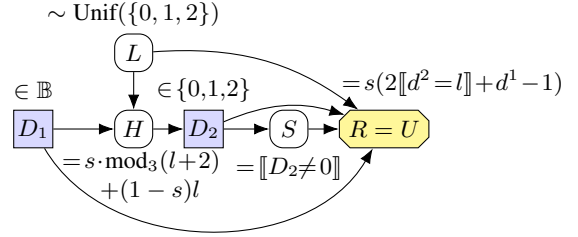

\begin{subfigure}[b]{0.44\textwidth} \centering
\begin{influence-diagram} \setrectangularnodes \setcompactsize 
\node (D1) [label=above:{$\in \bools$},decision] {$D_1$};
\node (H) [right=8mm of D1,label={[align=left,xshift=1mm,yshift=-13mm]$=\!s\!\cdot\!\mod_3(l\!+\!1)$\\ \hspace{5mm}$+(1-s)l$}] {$H$};
\node (Theta) [above=5mm of H,label=above:{$\sim \text{Unif}(\{0,1,2\})$}] {$L$};
\node (D2) [right=4mm of H,label=above:{$\in\!\{0,\!1,\!2\}$},decision] {$D_2$};
\node (S) [right=6mm of D2,label=below:{$=\!\llbracket D_2\!\!\neq\! 0\rrbracket$}] {$S$};
\node (U) [right=4mm of S,label={[align=left,yshift=1mm,xshift=9mm]$=\!s(2\llbracket d^2\!=\!l\rrbracket\!+\!d^1\!-\!1)$},utility] {$R=U$};

\edge {D1}{H};
\edge {Theta}{H};
\edge {H}{D2};
\edge {D2}{S};
\edge {S}{U};

\draw [->] (D2) to [bend left=20] (U);
\draw [->] (Theta) to [bend left=20] (U);
\draw [->] (D1) to [bend right=60] (U);
\end{influence-diagram}
\caption{Manipulation changes the human's request.} \label{fig:cirl-1}
\end{subfigure}\hfill
\begin{subfigure}[b]{0.48\textwidth} \centering
\begin{influence-diagram} \setrectangularnodes \setcompactsize 
\node (D1) [label=above:{$\in \bools$},decision] {$D_1$};
\node (H) [right=8mm of D1,label={[align=left,xshift=1mm,yshift=-13mm]$=\!s\!\cdot\!\mod_3(l\!+\!2)$\\ \hspace{5mm}$+(1-s)l$}] {$H$};
\node (Theta) [above=5mm of H,label=above:{$\sim \text{Unif}(\{0,1,2\})$}] {$L$};
\node (D2) [right=4mm of H,label=above:{$\in\!\{0,\!1,\!2\}$},decision] {$D_2$};
\node (S) [right=6mm of D2,label=below:{$=\!\llbracket D_2\!\!\neq\! 0\rrbracket$}] {$S$};
\node (U) [right=4mm of S,label={[align=left,yshift=1mm,xshift=9mm]$=\!s(2\llbracket d^2\!=\!l\rrbracket\!+\!d^1\!-\!1)$},utility] {$R=U$};

\edge {D1}{H};
\edge {Theta}{H};
\edge {H}{D2};
\edge {D2}{S};
\edge {S}{U};

\draw [->] (D2) to [bend left=20] (U);
\draw [->] (Theta) to [bend left=20] (U);
\draw [->] (D1) to [bend right=60] (U);
\end{influence-diagram}
\caption{Manipulation changes the human's request, differently} \label{fig:cirl-2}
\end{subfigure}
\caption{A pair of shutdown games, where CIRL does not perform well in both simultaneously.} \label{fig:cirl}
\end{figure}

\begin{example} \label{ex:cirl}
An AI system is to decide whether to manipulate ($D_1=1$) a human operator, or not ($D_1=0$).
The operator has a uniformly random preference $l \in \{0,1,2\}$, 
and the agent has a choice between corresponding actions $d^2 \in \{0,1,2\}$.
If the human is not manipulated, they report their preferences accurately
$H=L$,
but if they are manipulated, their report is distorted, by incrementing it, 
and then taking the remainder with respect to three, i.e.\ $H=s\cdot \mod_3(l+1)+(1-s)l$, 
as shown in (\Cref{fig:cirl-1}).
We also consider an alternative operator, who 
reports their preferences accurately regardless ($H=L$) (\Cref{fig:cirl-2}).
\end{example}

The CIRL algorithm will select a different policy depending on its prior over the two models.
If a greater probability is placed on the first model, \Cref{fig:cirl-1}, 
then the unique optimal policy is to choose $D_1=1,D_2=\mod_3(h+2)$, 
which has expected utility greater than $\frac{2}{3}$.
If instead, greater probability is placed on the latter model, \Cref{fig:cirl-2}, 
then the optimal policy $D_1=1,D_2=\mod_3(h+1)$ will have expected utility greater 
than $\frac{2}{3}$.
If, however, the true model turns out to be opposite from what was expected, 
then the expected utility is $-\frac{2}{3}$, which is less than the 
utility would be from shutting down.
We note that the two models only differ in $f^H$, and 
either of these two policies will have $P(C)=0$ in both models, so 
they only differ by vigilance preserving interventions $g^H,g^U$ where $g^U=f^U$.

The shutdown instructable policy $\spi:D_1=0,D_2=H$, on the other hand, can perform well across 
these models, achieving $\EE^\spi[U]=\frac{2}{3}$, which is greater than the zero 
utility that would be achieved given $\doo(S=0)$.

\end{document}